\documentclass[journal]{IEEEtran}
\usepackage{setspace}
\usepackage{balance}
\usepackage{caption}

\ifCLASSOPTIONcompsoc
    \usepackage[caption=false, font=normalsize, labelfont=sf, textfont=sf]{subfig}
\else
\usepackage[caption=false, font=footnotesize]{subfig}
\fi

\usepackage{etoolbox}
\usepackage{color}
\usepackage{amsmath,amssymb,bm}
\usepackage{amsthm}
\usepackage{xpatch}
\makeatletter
\xpatchcmd{\proof}{\@addpunct{.}}{\@addpunct{:}}{}{}
\makeatother

\usepackage{amsfonts,dsfont,color,bbm}
\usepackage[normalem]{ulem}

\usepackage{booktabs}
\usepackage{multirow}
\usepackage{multicol}
\usepackage{cite}
\usepackage{graphicx}
\usepackage{hyperref}

\renewcommand{\L}{\mathcal{L}}

\newcommand{\F}{\mathcal{F}}

\newcommand{\h}{\vec{h}}
\renewcommand{\k}{^{(k)}}

\newcommand{\W}{\vec{W}}
\newcommand{\V}{\vec{V}}

\renewcommand{\b}{\vec{b}}
\renewcommand{\vec}[1]{\mbox{\boldmath${#1}$}}

\newcommand{\X}{\vec{X}}

\renewcommand{\and}{ \,\mathrm{and}\, }
\newcommand{\real}{\mathbb{R}}

\renewcommand{\th}{^{\mathrm{th}}}
\renewcommand{\P}{\mathcal{P}}

\newcommand{\card}[1]{\vert{#1}\vert}

\newcommand{\x}{\vec{x}}
\renewcommand{\xi}{\vec{x}^{(i)}}

\DeclareMathOperator*{\argmax}{arg\,max}

\theoremstyle{plain}

\newtheoremstyle{customremark}
{}                
{}                
{}        
{}                
{\itshape\bfseries}       
{.}               
{ }               
{}                

\theoremstyle{customremark}
\newtheorem{remark}{Remark}
\theoremstyle{definition}
\newtheorem{definition}{Definition}

\forcsvlist{\listadd\newcitations}{cambria2019computational, cambria2012hourglass, susanto2020hourglass, cambriaieee, lo2017multilingual, senticnet6, senticnet5, kouloumpis2011twitter, balahur2014comparative, babelsenticnet, chung2014empirical, schuster1997bidirectional, demvsar2006statistical}

\begin{document}

\title{Multi-Label Sentiment Analysis on 100 Languages with Dynamic Weighting for Label Imbalance}

\author{Selim~F.~Yilmaz,~E.~Batuhan~Kaynak,~Aykut~Ko\c{c},~\IEEEmembership{Senior Member,~IEEE,}\\Hamdi~Dibeklio\u{g}lu,~\IEEEmembership{Member,~IEEE}~and~Suleyman~S.~Kozat,~\IEEEmembership{Senior Member,~IEEE}
\thanks{This work is supported in part by Outstanding Researcher Programme
Turkish Academy of Sciences. This work is also partially supported by Vodafone Turkey within the framework of 5G and Beyond Joint Graduate Support Programme coordinated by Information and Communication Technologies Authority.}
\thanks{S.~F.~Yilmaz, A.~Ko\c{c} and S.~S.~Kozat are with the Department of Electrical and Electronics
Engineering, Bilkent University, 06800 Ankara, Turkey (e-mail: \{syilmaz,kozat\}@ee.bilkent.edu.tr, aykut.koc@bilkent.edu.tr).
}
\thanks{A~Ko\c{c} is also with the National Magnetic Resonance Center UMRAM, 06800 Ankara, Turkey.}
\thanks{S.~S.~Kozat is also with the DataBoss A.S., Bilkent Cyberpark, 06800 Ankara, Turkey (e-mail: serdar.kozat@data-boss.com.tr).}
\thanks{B.~Kaynak~and~H.~Dibeklio\u{g}lu are with the Department of Computer Engineering, Bilkent University, 06800 Ankara, Turkey (e-mail: batuhan.kaynak@bilkent.edu.tr, dibeklioglu@cs.bilkent.edu.tr).
}
}

\maketitle

\begin{abstract}
We investigate cross-lingual sentiment analysis, which has attracted significant attention due to its applications in various areas including market research, politics and social sciences. In particular, we introduce a sentiment analysis framework in multi-label setting as it obeys Plutchik's wheel of emotions. We introduce a novel dynamic weighting method that balances the contribution from each class during training, unlike previous static weighting methods that assign non-changing weights based on their class frequency. Moreover, we adapt the focal loss that favors harder instances from single-label object recognition literature to our multi-label setting. Furthermore, we derive a method to choose optimal class-specific thresholds that maximize the macro-f1 score in linear time complexity. Through an extensive set of experiments, we show that our method obtains the state-of-the-art performance in 7 of 9 metrics in 3 different languages using a  single model compared to the common baselines and the best-performing methods in the SemEval competition. We publicly share our code for our model, which can perform sentiment analysis in 100 languages, to facilitate further research.
\end{abstract}

\section{Introduction}

\subsection{Preliminaries}
\label{sec:introduction}
\IEEEPARstart{W}{e} study sentiment analysis problem in multi-label setting, which has been widely studied in the literature due to its significance in various applications including market research, politics, public health and disaster management~\cite{SemEval2018Task1,zhusentivec,liu2012sentiment}. In particular, we introduce a method for cross-lingual sentiment analysis, which is a harder problem than the standard sentiment analysis problem since one needs to make predictions for various languages, including even unseen ones. Cross-lingual sentiment analysis aims to leverage high-quality and abundant resources in English for classification to improve the classification performance of resource-scarce languages~\cite{wang2020coarse}. Moreover, we employ data of three languages to obtain the best score in 7 of 9 metrics of Arabic, English and Spanish languages in the SemEval emotion classification~\cite{SemEval2018Task1}.

Emotions are an integral part of human communication and decision making mechanisms~\cite{cambria2019computational}. Plutchik~\cite{plutchik1980general}, in 1980, has created the wheel of emotions in his psychoevolutionary theory of emotion to illustrate his idea of emotion. He suggests eight bipolar primary emotions that appear on the opposite sides of the wheel: joy versus sadness, anger versus fear, disgust versus trust, surprise versus anticipation. The primary emotions are expressed at different intensities and the intermediate emotions occur as a mix of these primary emotions. Moreover, the emotions are non-exclusive in Plutchik's model as their combinations derive other emotions. There exist correlations between the emotions, e.g., joy and sadness are represented as the opposite emotions. The Hourglass of Emotions~\cite{cambria2012hourglass,susanto2020hourglass} reinterprets Plutchik's model by reorganizing the primary emotions along four independent dimensions. Following these emotion models, we formulate the sentiment analysis as the multi-label classification task, in which more than one label can be assigned to a text simultaneously. Yet, the class imbalance is an inherent issue in multi-label classification~\cite{xu2019survey}. Although the class imbalance has been extensively studied for the binary classification setting, it remains a challenge in multi-label classification
~\cite{xu2019survey}. Furthermore, the tail labels, i.e., the labels with a low number of instances, impact the performance significantly less compared to the common labels when the classes are equally weighted in the multi-label setting due to the rarity of relevant examples and result in suboptimal performance~\cite{wei2019does}. Thus, we introduce a dynamic weighting method to dynamically adjust the class weights during training to remedy the class imbalance.

In this article, we introduce a multilingual sentiment analysis framework in multi-label setting on 100 different languages. Our method uses focal loss to enhance the importance of hard examples. We introduce a dynamic weighting method to cope with the label imbalance. We also derive a macro-f1 maximization method within linear time complexity. Our method achieves the best result for 7 out of 9 metrics for the SemEval competition for Arabic, English and Spanish languages~\cite{SemEval2018Task1}. We also demonstrate the performance of our method on cross-lingual combinations of the datasets and assess the performance gains obtained by the components in our method.

\subsection{Prior Art and Comparisons}
\label{sec:relatedwork}
Current sentiment analysis methods can be represented under three different categories: knowledge-based (symbolic), statistical (sub-symbolic), and hybrid~\cite{cambriaieee}. Knowledge-based methods mainly use manually or semi-automatically constructed lexicons in order to classify texts into affect categories~\cite{cambriaieee,lo2017multilingual}. Knowledge-based methods are widely adopted because of their accessibility and affordability~\cite{cambriaieee}. Single language sentiment analysis usually employs lexicons or ontologies, which are manually or semi-automatically constructed~\cite{lo2017multilingual}. SenticNet is a knowledge-base of sentimental concepts, aiming to improve the generalization and interpretability of sentiment analysis methods. Statistical approaches for sentiment analysis task include deep learning based methods, which we describe in the next paragraph. Hybrid methods adopt both knowledge-based approaches and statistical methods. SenticNet 6~\cite{senticnet6} ensembles both symbolic and sub-symbolic approaches by integrating logical reasoning into deep learning framework.

Deep learning based methods have been shown to be successful in various classification tasks~\cite{lecun2015deep,young2018recent}. SenticNet 5~\cite{senticnet5} shows that concept clusters, which share similar syntactic or semantic functions, can automatically be discovered through a deep learning framework. Transfer learning approaches have been popular in sentiment analysis and shown to be successful, especially in datasets with a small number of instances~\cite{deepmoji}. Through transfer learning, large unlabeled corpora in social media have been incorporated to increase the target sentiment analysis task's performance, e.g., \cite{deepmoji} employs 1.7 billion tweets with emojis to pretrain the network. However, \cite{kant2018practical} demonstrates that the transfer learning approach does not improve the performance on the SemEval emotion classification competition datasets, which is our target due to the richness of its labels, which has significantly more number of instances compared to the number of instances used in~\cite{deepmoji}. \cite{suttles2013distant} emulates a multi-label classifier through a binary classifier for each of the four opposite emotions that are on the opposite sides of Plutchik's wheel of emotions, such as joy and sadness. However, their approach does not include the correlations to the rest of the labels since they train each of the four classifiers with the objective of binary classification of the opposite sides. To remedy those issues, we introduce a multi-label deep learning model for the emoji prediction task that directly predicts the active set of labels simultaneously, i.e., in the multi-task setting.
Moreover, the multi-label classification is a generalization of binary and multi-class classification tasks as we describe through remarks. Thus, our method is also applicable to these tasks. Multi-label classification also requires a prediction method that converts scores into predictions, for which we derive a class-specific thresholding method by macro-f1 maximization in linear time complexity.

Multilingual sentiment analysis on social media has become an important problem as it is becoming mainstream media for communication and expression of thoughts in many different languages~\cite{lo2017multilingual}. Being easily applicable to other languages, \cite{kouloumpis2011twitter} leverages hashtags to collect sentiment classification dataset and performs training on it. \cite{balahur2014comparative} translates an annotated English dataset to build classifiers in other languages. However, this method is prone to errors caused by translation.  BabelSenticNet employs statistical machine translation to provide a concept-level multilingual sentiment analysis framework that applies to a wide range of domains~\cite{babelsenticnet}. Our method leverages XLM-RoBERTa~\cite{conneau2019unsupervised} to perform cross-lingual and multilingual sentiment classification on tweets.

Multi-label classification has an inherent issue of data imbalance~\cite{xu2019survey}. Although significant research has been performed in the literature, the class imbalance problem remains a challenge for multi-label classification~\cite{xu2019survey}. Consider the multi-label classification task with $16$ distinct labels. There are $2^{16}$ possible combinations in the superset of the labels. Accordingly, it is not feasible to obtain balanced data for each combination of the labels. Many studies in multi-output classification either try to balance the data by resampling or ignore the imbalance~\cite{xu2019survey}. Yet, the over-sampling and under-sampling methods are not designed for the multi-label classification; thus, their adaptation to the multi-label setting is not straightforward~\cite{xu2019survey}. One heuristic that is widely adapted uses inverse class frequency per class as a weighting factor~\cite{huang2016learning}. However, this heuristic results in suboptimal performance as shown by~\cite{cui2019class} and in Section~\ref{sec:ablation_loss}. \cite{cui2019class} replaces the inverse number of instances with the expected volume of instances and a controlling hyperparameter. \cite{aurelio2019learning} proposes to use class prior probabilities as weights for the cross-entropy loss. Commonly, these methods propose static weights for each class. To remedy the label imbalance in the multi-label setting, we introduce a novel dynamic weighting method, which equalizes the contribution of each class to the loss. We use focal loss~\cite{lin2017focal} to incorporate the hardness of the instances and our dynamic weighting method can readily be adapted to other losses as we show through a remark.

Recent language models such as BERT~\cite{BERT} have been dominating the areas in the NLP literature; however, they contain an excessive amount of parameters. Accordingly, training or fine-tuning these models require an excessive amount of resources~\cite{BERT}. We employ XLM-RoBERTa~\cite{conneau2019unsupervised}, which is a robustly trained BERT on 100 languages, as feature extractor to benefit from BERT as well as reducing the number of required resources.

SemEval emotion classification competition~\cite{SemEval2018Task1} has paved the way for many multi-label sentiment analysis models. \textit{EMA}, \textit{PARTNA} are among the models that opt for the more traditional support vector machine approaches and still achieve the best results in the Arabic language~\cite{badaro2018ema}. On the other hand, more recent long short-term memory (LSTM), convolutional neural network (CNNs) and attention models are also adopted to obtain the highest ranked results in English and Spanish~\cite{baziotis2018ntua, gee2018psyml}. It is important to note that most of these models are language specific and use special embeddings such as AraVec~\cite{aravec} or special lexicons paired with language specific preprocessing steps. \textit{Tw-StAR} attempts to create a generic model to apply multiple languages, yet, is ranked behind the language-specific models~\cite{mulki2018tw}. We introduce a framework that uses bidirectional LSTM with attention and multi-label focal loss, which achieves the best score only using a single model on 7 of the 9 metrics on three different languages of the SemEval emotion classification competition~\cite{SemEval2018Task1}.

\subsection{Contributions}
Our contributions are as follows:
\begin{enumerate}
    \item First time in the literature to the best of our knowledge, we introduce a multi-label emotion classification method capable of producing uniformly high classification performance on 100 different languages using a single model. Our method can readily be adapted to the cross-lingual platforms such as Amazon without using any language detection component. We make our model publicly available\footnote{\label{fn:github}\href{https://github.com/selimfirat/multilingual-sentiment-analysis}{https://github.com/selimfirat/multilingual-sentiment-analysis}} to facilitate reproducibility and further research.
    \item We introduce a dynamic weighting method to remedy the class imbalance that is an inherent problem in multi-label classification with adaptive loss weights as training progress, unlike the previous static weighting methods~\cite{aurelio2019learning,cui2019class}. We demonstrate significant performance gains compared to the previous weighting methods. Our method performs no worse than the uniform weighting, i.e., no weighting, for none of the hyperparameter choices. Our dynamic weighting method can be readily extended to other losses as we show through a remark.
    \item We derive a method to maximize macro-f1 with class-specific threshold choices, which reduces the time complexity from exponential to linear. Moreover, we adapt focal loss to our multi-label emotion classification framework from the single-label object recognition literature, and we show performance improvements obtained via the focal loss~\cite{lin2017focal}.

    \item Through an extensive set of experiments, we show that our model achieves the best scores in 7 out of 9 metrics in the SemEval emotion classification competition for Arabic, English and Spanish via a single model~\cite{SemEval2018Task1}. Furthermore, we perform cross-lingual experiments, hyperparameter studies and an ablation study to assess the effectiveness of our method.

\end{enumerate}

\subsection{Organization of this Paper}
The rest of the paper is organized as follows. In Section~\ref{sec:problem_description}, we describe the multi-label sentiment analysis task and show that it is the generalization of the binary and multi-class classification tasks. In Section~\ref{sec:methodology}, we introduce our deep metric learning based framework and the components to cope with the label imbalance. In Section~\ref{sec:experiments}, we demonstrate the performance improvements obtained by our proposed model compared to the state-of-the-art methods in the literature and the SemEval~\cite{SemEval2018Task1} emotion classification competition winners. In Section~\ref{sec:conclusion}, we conclude by providing remarks.

\section{Problem Description}
\label{sec:problem_description}
In this article, all vectors are column vectors and defined by boldfaced lowercase letters. All matrices and tensors are represented by boldfaced uppercase letters. $\card{\cdot}$ denotes cardinality, i.e., the number of elements, of set $\cdot$.

We aim to predict the labels of the given text in multi-label framework through our network $\F$. We receive training data $\P=\{ (s_i, \vec{c}_i) \}_i^n$, where $s_i$ is the text of $i\th$ training instance, $n$ is the number of training instances, $\vec{c}_i = [c_{i,1} \, c_{i,2} \, ..., \, c_{i,w} ]^T$ is the label vector of the $i\th$ training instance, $w$ is the number of classes and  $c_{i,a}$, $a \in \{1, 2, ...w\}$, is defined by
\begin{align*}
    c_{i,a} = 
\begin{cases} 
1,  &\text{ if class $a$ is inferred} \\
0, & \text{  otherwise}.
\end{cases}
\end{align*}
To satisfy this decision function, we predict score $\hat{c}_{i,a}$ for the target sentence $s_i$ via our network $\F$ as:
\begin{equation}
\hat{c}_{i,a} = \F(s_i) = p(c_{i,a}=1\mid s_i).
\label{eq:problem_class_label}
\end{equation}

\begin{remark}
Multi-class classification is a generalization of multi-class and binary classification tasks. For both, we have only one active label, i.e., $\sum_{a=1}^w c_{i,a} = 1$, $\forall i \in \{1, 2, ..., n \}$. The number of classes $w=2$ and $w>2$ for binary and multi-label classification, respectively. Since we formulate the problem as a multi-label classification, our framework is applicable to binary classification, multi-class classification and multi-label classification settings.
\end{remark}
\section{Methodology}
\label{sec:methodology}
In this section, we first describe the language modeling and recurrent modeling with attention for the multi-label classification. We then introduce our multi-label adaptation of focal loss and our dynamic weighting method. Lastly, we derive a method to select thresholds by maximizing macro-f1 within linear time complexity. Fig.~\ref{fig:lstm} illustrates the overall structure of our methodology.
\begin{figure}[tbp!]
    \centering
    \includegraphics[width=\linewidth]{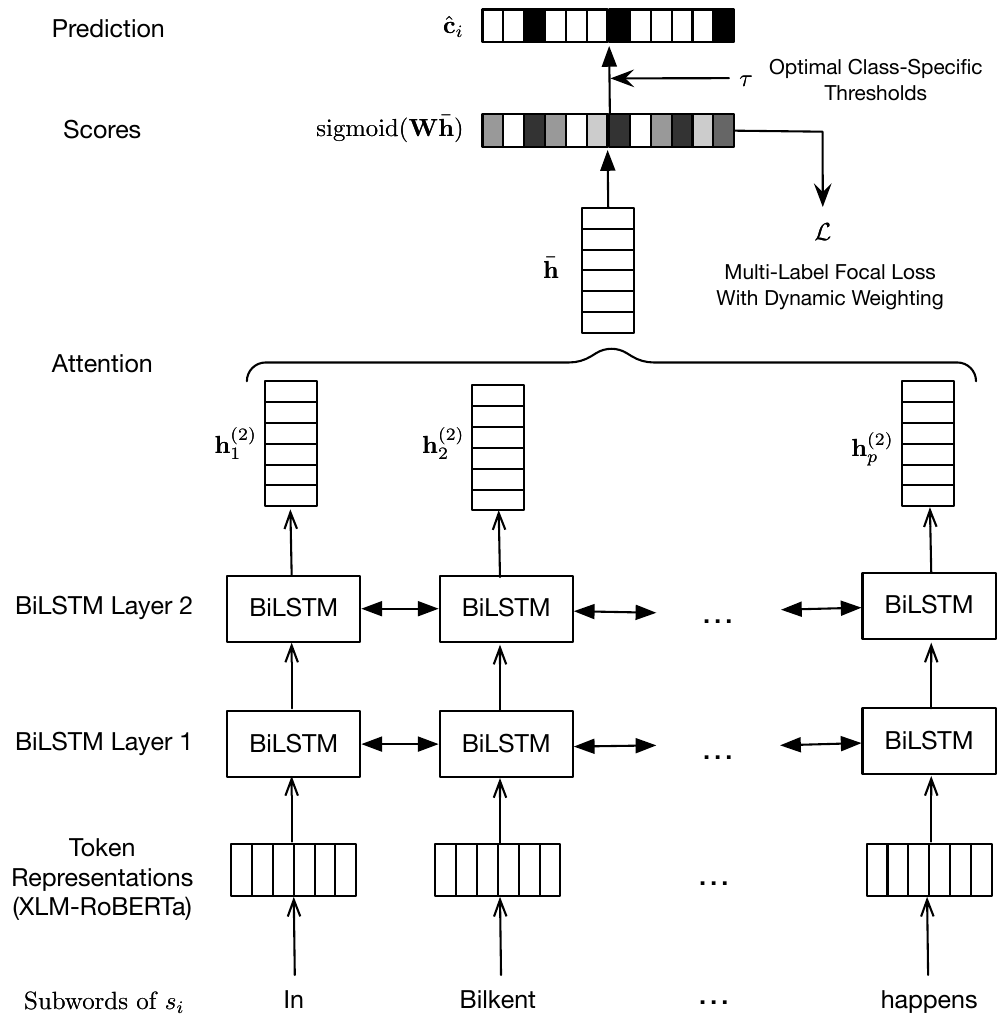}
    \caption{The overall structure of our model.}
    \label{fig:lstm}
\end{figure}

\subsection{Deep Multilingual Language Modeling}
Here, we describe our language modeling approach using XLM-RoBERTa~\cite{conneau2019unsupervised}.

Traditional approaches such as well known Bag-of-Words fail to generalize to the unseen data due to the sparsity of the language~\cite{collobert2011natural}. Early word embedding-based methods, such as the well-known word2vec~\cite{mikolov2013distributed}, based approaches have been used to cope with this problem via learning a vector for each word in a large vocabulary exploiting semantic relations between words~\cite{collobert2011natural}. However, these methods assign a single vector to each word regardless of the context of the target sentence. Recently, language models such as BERT have achieved outstanding results on various tasks~\cite{BERT}. These language models assign context-dependent vectors to each token in the target space instead of assigning a fixed vector. These models are trained using large corpora in an unsupervised setting. However, these models contain millions of parameters, and it is not reasonable to finetune them on a small corpus. Thus, we use feature vectors extracted from the pretrained model for each text instead of directly finetuning the pretrained model.

As shown in Fig.~\ref{fig:lstm}, we use XLM-RoBERTa pretrained tokenizer and pretrained model~\cite{conneau2019unsupervised}. XLM-RoBERTa is pretrained on CommonCrawl corpora of 100 different languages. We first tokenize the input sentence $s_i$ into subword units via byte-pair encoding using Sentencepiece Tokenizer~\cite{kudo2018sentencepiece}. We convert the sentence $s_i$ to $\X^{(i)} \in \real^{m \times d_i}$ by using the hidden state vectors of the pretrained language model, where $m$ is the embedding vector length and $d_i$ is the number of tokens in the sentence $s_i$. We obtain an embedding vector for each token in the sentence, i.e., $\X^{(i)} = [\x_1 \, ..., \, \x_{d_i}]$, where $\x_{j} \in \real^m \, , \forall j \in \{ 1, 2, ..., d_i\}$.

\begin{remark}
Our model can be adapted to other languages since we tokenize via byte-pair encoding and convert to features without applying any language-dependent preprocessing. For the languages that XLM-RoBERTa does not support, one can directly use any other pretrained model that supports the target language. We show the cross-lingual performance of our method in Section~\ref{sec:cross_lingual_experiments}.
\end{remark}

\subsection{Temporal Modeling of Sentence via Recurrent Networks}
Here, we describe our recurrent modeling for the multi-label emotion classification using the frozen features via the language modeling network.

We are given a sequence of token embeddings $\X^{(i)} \in d_i \times m$ for the sentence $s_i$, where $d_i$ is the number of tokens in the sentence $s_i$ and $m$ is the embedding size. $\x_k \in \real^m$ indicates the embedding of the $k\th$ token.

As shown in Fig.~\ref{fig:lstm}, we use bidirectional RNNs to incorporate both the forward and the backward information of the sequence. Through the RNN we process the variable length sequences. We employ deep networks, where the number of layers is $u$. For timestep $t$ and $k\th$ layer, we utilize $\overrightarrow{\h}\k_t$ and $\overleftarrow{\h}\k_t$ notations to define forward and backward RNNs, respectively. We define $k\th$ layer of the forward RNN that uses Elman's formulation~\cite{elman} as:
\begin{align*}
    \overrightarrow{\h}\k_t = \mathrm{tanh}(\W\k_{\mathrm{hh}} \overrightarrow{\h}\k_{t-1} + \W\k_{\mathrm{hx}} \overrightarrow{\h}^{(k-1)}_t + \b\k),
\end{align*}
where $\overrightarrow{\h}^{(0)}_t = \x_t$ for $t \in \{1,2, ...,d_i\}$ $\overrightarrow{\h}\k_0 \sim \mathcal{N}(0, 0.01)$, $\vec{b}\k$ is the bias term to be learned and $\W_{\mathrm{hh}}\k$, $\W_{\mathrm{hx}}\k$ are the weights to be learned. We also define the backward RNN's hidden state $\overleftarrow{\h}\k_t$ for $k\th$ layer by feeding the reversed input to the RNN, i.e.,
\begin{align*}
    \overleftarrow{\h}\k_t = \mathrm{tanh}(\V\k_{\mathrm{hh}} \overleftarrow{\h}\k_{t+1} + \V\k_{\mathrm{hx}} \overleftarrow{\h}^{(k-1)}_t + \vec{c}\k),
\end{align*}
where $\overleftarrow{\h}^{(0)}_t = \x_{d_i-t+1}$, $\overleftarrow{\h}\k_{d_i} \sim \mathcal{N}(0, 0.01)$, $\vec{c}\k$ is the bias term to be learned and $\V_{\mathrm{hh}}\k$, $\V_{\mathrm{hx}}\k$ are the weights to be learned. 

\begin{remark}
We extend our framework to the LSTM~\cite{hochreiter1997long} due to its success in capturing complex temporal relations. We feed the input sentence embedding $\X^{(i)}$ to the LSTM instead of the RNN as:
\begin{align*}
\vec{z}_t\k &= \mathrm{tanh}(\vec{W}\k_z \overrightarrow{\h}_t^{(k-1)} + \vec{V}\k_z \overrightarrow{\vec{h}}\k_{t-1} + \vec{b}_z\k) \\
\vec{s}_t\k &= \mathrm{sigmoid}(\vec{W}\k_s \overrightarrow{\h}_t^{(k-1)} + \vec{V}\k_s \overrightarrow{\vec{h}}\k_{t-1} + \vec{b}_s\k) \\
\vec{f}_t\k &= \mathrm{sigmoid}(\vec{W}\k_f \overrightarrow{\h}_t^{(k-1)} + \vec{V}\k_f \overrightarrow{\vec{h}}\k_{t-1} + \vec{b}\k_f) \\
\vec{c}_t\k &= \vec{s}_t\k \odot \vec{z}_t\k + \vec{f}_t\k \odot \vec{c}_{t-1}\k \\
\vec{o}_t\k &= \mathrm{sigmoid}(\vec{W}\k_o\overrightarrow{\h}_t^{(k-1)}+\vec{R}\k_o \overrightarrow{\vec{h}}_{t-1}\k+ \vec{b}\k_o) \\
\overrightarrow{\vec{h}}_{t}\k &= \vec{o}_t\k \odot \mathrm{tanh}(\vec{c}_t\k), 
\end{align*}
where $\overrightarrow{\h}_t^{(0)}=\x_t$, $\overrightarrow{\h}_0\k \sim \mathcal{N}(0, 0.01)$, $\vec{c}_t\k \in \real^m$ is the cell state vector, $\vec{h}_t\k \in \real^w$ is the hidden state vector, for the $t$\textsuperscript{th} LSTM unit. $\vec{s}_t\k$, $\vec{f}_t\k$ and $\vec{o}_t\k$ are the input, forget and output gates, respectively. $\odot$ is the operation for elementwise multiplication. $\vec{W}$, $\vec{V}$, and $\vec{b}$ with the subscripts $z$, $s$, $f$, and $o$ are the parameters of the LSTM unit to be learned. We also define the backward LSTM via $\overleftarrow{\h}_t\k$ by reversing the input order for each layer of the LSTM, as in RNNs.
\end{remark}

We concatenate the hidden states of the backward  and the forward RNN of $k\th$ layer at time $t$ as:
\begin{align*}
    \h\k_t = \begin{bmatrix}
    \overrightarrow{\h}\k_t \\ \overleftarrow{\h}\k_t
    \end{bmatrix}.
\end{align*}
We then apply attention to the hidden states by weighing each timestep's hidden state with a single parameter as~\cite{yang2016hierarchical}:
\begin{align*}
    \Bar{\h} = \sum_{t=1}^p \beta_t \h_t^{(u)},
\end{align*}
where $p$ is the sequence length and $\beta_t = \frac{\exp \left ( \h_t \vec{s} \right )}{\sum_{i=1}^{p} \exp \left (\h_i \vec{s}\right )}$ for the timestep $t \in \{1,2,...,p\}$. Lastly, we use linear layer and sigmoid activation to convert our predictions to the labels as:
\begin{equation}
    \vec{r} = \mathrm{sigmoid}(\vec{W} \Bar{\h}),
\label{eq:sigmoid_output}
\end{equation}
where $\vec{r} \in \real^s$ and $s$ is the number of the target labels of the task. 

\begin{remark}
We use sigmoid activation at the final layer instead of softmax since the softmax assumes independence between labels, whereas in our case, the labels are non-independent due to Plutchik's theory~\cite{plutchik1980general} as we describe in Section~\ref{sec:introduction}.
\end{remark}

\subsection{Multi-Label Focal Loss}
\label{sec:loss}

In this section, we adapt the focal loss for our multi-label framework from the single-label object recognition literature~\cite{lin2017focal}. We define $p_{i, a}$ for notational convenience as the following:
\begin{align*}
    p_{i, a} = 
    \begin{cases} 
        r_{i,a}, & \mathrm{if } \,  c_{i,a}=1 \\
        1 - r_{i,a}, & \mathrm{otherwise},
    \end{cases}
\end{align*}
where $r_{i,a}$ is the sigmoid output for class $a$ and the instance $i$, which is obtained via~\eqref{eq:sigmoid_output}.
Then, the standard cross-entropy loss for instance $i$ and class $a$ becomes $- \log p_{i,a}$.

Focal loss has been proposed to overcome the class imbalance problem in object recognition, which extends the cross-entropy loss~\cite{lin2017focal}. The focal loss focuses training of the hard instances instead of the well-classified ones as:
\begin{align*}
    l_{i,a} = -(1 - p_{i,a})^\gamma \log p_{i,a},
\end{align*}
where $\gamma \in \real$ is a tunable parameter  and $\gamma \geq 0$. Notice that the focal loss extends the cross-entropy loss by multiplying with $(1-p_{i,a})^\gamma$.

We convert the loss into a scalar by taking weighted sum w.r.t. the classes and averaging w.r.t. the instances in the batch as:
\begin{equation}
    \L = \frac{1}{b} \sum_{i=1}^b \sum_{a=1}^w \alpha_{t,a} l_{i,a},
    \label{eq:loss}
\end{equation}
such that $\sum_{a=1}^w \alpha_{t,a} = 1$, $t$ is the index of the mini-batch iteration, $b$ is the batch size and $\alpha_{t,a}$ is the weight of the class $a$ at the mini-batch iteration $t$. We can assign equal weights to by setting $\alpha_{t,a}=\frac{1}{w}$ for each class $a$ and for all mini-batch iteration $t$. In the following section, we introduce a novel method for choosing $\alpha_{t,a}$ to remedy the class imbalance.

\subsection{Novel Dynamic Weighting Method for Label Imbalance}
Here, we introduce our dynamic weighting method to improve the imbalanced multi-label classification, which can also be applied to the single-label problems and other loss functions, as we show through remarks.

Although focal loss improves the imbalanced classification performance, there is still plenty of room for improvement. For instance, \cite{lin2017focal} uses an alpha balanced variant of the focal loss in practice, where they choose the inverse frequency of the class as in the imbalanced classification. \cite{cui2019class} also extends focal loss by class volume based formulation and introduces another hyperparameter. We introduce a method to equalize the losses from all classes in the problem. Our goal is to define weights in a way that each class has an equal contribution to the loss, i.e.,
\begin{align*}
    \sum_{i=1}^{\card{\P}} \alpha_{t,1} l_{i,1} = \sum_{i=1}^{\card{\P}} \alpha_{t,2} l_{i,2} = ... = \sum_{i=1}^{\card{\P}} \alpha_{t,w} l_{i,w} ,
\end{align*}
where $\P$ is the training data. 

Finding the exact value for $\alpha_{t,a}$ is intractable since model parameters change after each mini-batch and we train in mini-batch setting. Thus, we track the losses by exponentially smoothed approximation $\omega_{t, a}$ at mini-batch iteration $t$ and class $a$, which is given by
\begin{align*}
    \omega_{t, a} = (1-\kappa) \omega_{t-1, a} + \kappa \sum_{i=1}^b l_{i,a} ,
\end{align*}
where $\kappa$ is the smoothing hyperparameter to be tuned and $\omega_{1, a}=\frac{1}{w}, \forall a \in \{1,2,...,w\}$. We invert $\omega_{t,a}$ and introduce a very small $\epsilon$ term if there appears no loss for any class for numerical stability of our method since we may get $0$ loss for some classes, as the following:
\begin{align*}
    \phi_{t,a} = \frac{1}{\epsilon + \omega_{t, a}},
\end{align*}
where we set $\epsilon=1 \times 10^{-5}$. Using $\phi_{t,a}$, we define $\alpha_{t,a}$ in~\eqref{eq:loss} as:
\begin{equation}
    \alpha_{t,a} = \frac{\phi_{t,a}}{\sum_{u=1}^w \phi_{t,u}}.
    \label{eq:dynamic_weighting}
\end{equation}
Through~\eqref{eq:dynamic_weighting}, we gurantee that the weights sum up to $1$ for any mini-batch iteration. We set the gradient w.r.t. the network parameters $\Theta$ to zero, i.e., $\nabla_\Theta \alpha_{t,a}=0, \, \forall t\in \{1,2,...\}, a\in\{1,2,...,w\}$. We balance the loss contribution from the classes by using $\alpha_{t,a}$ in~\eqref{eq:loss}.

\begin{remark}
Notice that our dynamic weighting method's weights change over time w.r.t. the hardness of the instances among classes, unlike the previous methods in the literature~\cite{cui2019class,aurelio2019learning}.
\end{remark}

\begin{remark}
Dynamic weighting is loss-agnostic, thus, can readily be adapted to the other alternative losses. For example, we can adapt it into the cross-entropy loss by setting $l_{i,a}=-\log p_{i,a}$ and directly use~\eqref{eq:loss}.
\end{remark}
\begin{remark}
Dynamic weighting method can also be applied to the single label problems without any change since the multi-label problem is a generalization of the single-label variant.
\end{remark}

\subsection{Class-Specific Thresholding via Macro-F1 Maximization}
\label{sec:class_specific_thresholding}
We derive a macro-f1 maximization method by choosing the optimal class-specific threshold within linear time complexity.

We have the model output $\vec{\hat{c}_i} = \vec{r}_i$, which is our score vector, that is to be thresholded to make a prediction. We have a class-specific score $\hat{c}_{i,a}$ for class $a$. We expect high scores for the inferred classes and low scores for the non-inferred classes. We split a part of the validation set as the thresholding set and then use it to choose the optimal threshold hat maximizes the macro-f1 score. We concatenate the scores of all instances in the thresholding set $\mathcal{T}$ into $\hat{\vec{c}}_a \in \real
^{\card{\mathcal{T}}}$ as:
\begin{align*}
    \vec{\hat{c}}_a = \begin{bmatrix}
    \hat{c}_{1,a} &
    \hat{c}_{2,a} &
    ... &
    \hat{c}_{\card{\mathcal{T}},a} 
    \end{bmatrix}^T .
\end{align*}
Our aim is to find a threshold vector $\vec{\tau} \in \real^w$ given by
\begin{align*}
    \vec{\tau} = \begin{bmatrix}
    \tau_1 & \tau_2 & ... & \tau_w
    \end{bmatrix}^T .
\end{align*}
We select the optimal threshold for each class that maximizes the macro-f1 score, which is the F1 score calculated for each class and averaged among the classes. F1 is the harmonic mean of the precision and recall for a class $a$, i.e.,
\begin{align*}
    \mathrm{F1}(\vec{c}_a, \hat{\vec{c}_a}) =
    \frac{2 \times \mathrm{precision} \times \mathrm{recall}}{\mathrm{precision} + \mathrm{recall}}
\end{align*}
s.t.
\begin{align*}
    \mathrm{precision}(\vec{c}_a, \hat{\vec{c}_a}) &= \frac{\mathrm{TP}}{\mathrm{TP} + \mathrm{FP}}\\
    \mathrm{recall}(\vec{c}_a, \hat{\vec{c}_a}) &= \frac{\mathrm{TP}}{\mathrm{TP} + \mathrm{FN}},
\end{align*}
where $\mathrm{TP}$, $\mathrm{FP}$ and $\mathrm{FN}$ are the number of true positives, false positives and false negatives, respectively. 

\begin{definition}
The $\argmax_{\rho} f(\rho)$ function returns $\rho$ that maximizes the proceeding function $f(\rho)$. Note that if more than one value of $\rho$ maximize the function $f(\rho)$, then $\argmax_\rho f(\rho)$ returns the minimum $\rho$ among the ones maximizing $f(\rho)$.
\end{definition}

As shown in Fig.~\ref{fig:lstm}, we select the threshold vector $\vec{\tau}$ to threshold the model scores by maximizing the $\mathrm{MacroF1}$ function on the validation set as:
\begin{equation}
    \vec{\tau} = \argmax_{\vec{\tau}} \mathrm{MacroF1}(\delta(\hat{\vec{c}_a} \geq \tau_a), \vec{c}_a),
    \label{eq:argmax_macrof1}
\end{equation}
where $\delta(\cdot)$ function returns the same sized vector with its input, which outputs 1 for the dimensions that satisfy inequality and 0 for the rest. Directly optimizing~\eqref{eq:argmax_macrof1} via grid search becomes infeasible as the number of classes increases and it may not be possible to find the optimal value since the time complexity is in $\Theta ( T^w )$,\footnote{$\Theta(w(n))$ denotes the set of all $r(n)$, where $a_1w(n) \leq r(n) \leq a_2w(n)$, $\forall n > n_0$ for $n\in \mathbb{Z}^+$ such that there exist positive integers $a_1$, $a_2$, and $n_0$.} where $T$ is the number of elements to be tried and $w$ is the number of classes. By expanding the $\mathrm{MacroF1}$ function in~\eqref{eq:argmax_macrof1}, we obtain
\begin{align*}
    \vec{\tau} &= \argmax_{\vec{\tau}} \frac{1}{w} \sum_{a=1}^w \mathrm{F1}(\delta(\hat{\vec{c}_a} \geq \tau_a), \vec{c}_a). \\
    &= \frac{1}{w} \sum_{a=1}^w \argmax_{\vec{\tau}} \mathrm{F1}(\delta(\hat{\vec{c}_a} \geq \tau_a), \vec{c}_a). \\
\end{align*}
Since the class-specific thresholds in $\vec{\tau}$ are independent, we separate the thresholds into different $\argmax$ functions as:
\begin{align*}
    \tau_a &= \argmax_{\tau_a} \mathrm{F1}(\delta(\hat{\vec{c}_a} \geq \tau_a), \vec{c}_a) .
\end{align*}

Thus, we obtain that
$\vec{\tau}$ is equivalent to $ \begin{bmatrix} \tau_1 & \tau_2 & ... & \tau_w \end{bmatrix}^T $ such that
\begin{equation}
\tau_a = \argmax_{\tau_a} \mathrm{F1}(\delta(\hat{\vec{c}_a} \geq \tau_a), \vec{c}_a),
\label{eq:argmax_linear}
\end{equation}
for all $a \in \{1,2,...,w\}$.

Using~\eqref{eq:argmax_linear}, the time complexity becomes linear w.r.t. the number of classes, i.e., $\Theta (T w)$. Thus, we calculate the threshold vector $\vec{\tau}$ using~\eqref{eq:argmax_linear}.

\begin{remark}
Notice that it is not reasonable to choose the threshold using the training set since the model already memorizes it and unavoidably performs biased scoring for the training data. This is why we use the unseen thresholding set.
\end{remark}
\section{Experiments}
\label{sec:experiments}
In this section, we first describe the datasets, the evaluation methodology and the implementation details. We then compare our method with the first ranking methods in the SemEval emotion classification competition~\cite{SemEval2018Task1} and the state-of-the-art methods. We then analyze the performance of our method via cross-lingual experiments. Later, we demonstrate performance gains obtained via our dynamic weighting method and analyze its hyperparameter. Finally, we present our method's individual class performances and demonstrate the performance gains obtained by our method's components via an ablation study.

\subsection{Datasets}
\begin{figure}[htbp!]
    \centering
    \includegraphics{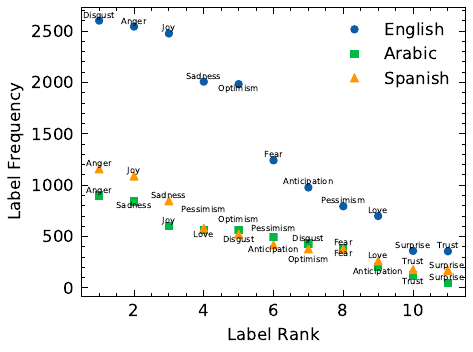}
    \caption{Number of label occurrences vs. rank plot of the labels in datasets that demonstrate the class imbalance.}
    \label{fig:datasets_zipf}
\end{figure}

We use datasets in three different languages from the SemEval competition~\cite{SemEval2018Task1}: SemEval-Arabic, SemEval-English and SemEval-Spanish. For simplicity, we refer SemEval-Arabic, SemEval-English and SemEval-Spanish datasets as Arabic, English and Spanish, respectively. Since the datasets are in multi-label setting, the instances contain zero or more labels among the 11 labels in the dataset. Fig.~\ref{fig:datasets_zipf} demonstrates the class imbalance via the number of occurrence vs. rank plot of the labels in the datasets. We use the splits of the SemEval emotion classification competition~\cite{SemEval2018Task1}. Arabic dataset has a total of 4,381 instances consisting of 160,206 tokens and split into 3,561 training, 679 validation and 2,854 test instances. The English dataset has 10,983 instances consisting of 338,763 tokens and split into 6,838 training, 886 validation and 3,259 test instances. The Spanish dataset has 7,094 instances consisting of 176,650 tokens and split into 2,278 training, 585 validation and 1,518 test instances.
\addtolength{\tabcolsep}{-1pt}
\begin{table*}[t!]
\centering
\begin{tabular}{lccccccccc}
& \multicolumn{3}{c}{Arabic} & \multicolumn{3}{c}{English} & \multicolumn{3}{c}{Spanish} \\
\cmidrule(lr){2-4}
\cmidrule(lr){5-7}
\cmidrule(lr){8-10}
Method & Macro-F1 & Micro-F1 & Jaccard & Macro-F1 & Micro-F1 & Jaccard & Macro-F1 & Micro-F1 & Jaccard \\ \midrule
Ours & {\bf 55.0} & {\bf 66.1} & {\bf 53.4} & {\bf 58.4} & 69.6 & 57.6 & {\bf 53.0} & {\bf 60.6} & {\bf 48.6} \\\midrule
Ours-SL & 51.3 & 57.5 & 44.3 & 56.4 & 68.7 & 56.5 & 50.5 & 56.6 & 45.3 \\
Tw-StAR~\cite{mulki2018tw} & 44.6 & 59.7 & 46.5 & 45.2 & 60.7 & 48.1 & 39.2 & 52.0 & 43.8 \\
FastText~\cite{fasttext} & 35.3 & 40.2 & 25.5 & 35.0 & 39.9 & 25.5 & 27.0 & 31.9 & 20.6 \\
XLM-RoBERTa~\cite{conneau2019unsupervised} & 48.2 & 65.2 & 53.0 & 56.1 & 69.3 & 57.3 & 48.3 & 57.0 & 47.3 \\
CA-GRU~\cite{samy2018context} & 49.5 & \underline{64.8} & \underline{53.2} & - & - & - & - & - & - \\
HEF-DF~\cite{alswaidan2020hybrid} & \underline{50.2} & 63.1 & 51.2 & - & - & - & - & - & - \\
EMA~\cite{badaro2018ema} & 46.1 & 61.8 & 48.9 & - & - & - & - & - & - \\
PARTNA & 47.5 & 60.8 & 48.4 & - & - & - & - & - & - \\
NTUA-SLP~\cite{baziotis2018ntua} & - & - & - & 52.8 & \underline{70.1} & \underline{58.8} & - & - & - \\
psyML~\cite{gee2018psyml} & - & - & - & \underline{57.4} & 69.7 & 57.4 & - & - & - \\
NVIDIA~\cite{kant2018practical} & - & - & - & 56.1 & 69.0 & 57.7 & - & - & - \\
DeepMoji~\cite{deepmoji} & - & - & - & 55.9 & 65.7 & 52.8 & - & - & - \\
BERT~\cite{BERT} & - & - & - & 54.7 & 69.3 & 57.4 & - & - & -\\
RoBERTa~\cite{liu2019roberta} & - & - & - & 57.4 & {\bf 71.0} & {\bf 59.1} & - & - & - \\
ELiRF-UPV~\cite{gonzalez2019elirf} & - & - & - & - & - & - & \underline{44.0} & 53.5 & 45.8 \\
MILAB\_SNU & - & - & - & - & - & - & 40.7 & \underline{55.8} & \underline{46.9} \\
\bottomrule
\end{tabular}
\caption{Comparison of our method with the state-of-the-art methods in the literature and the winners in SemEval-2018 emotion classification competition~\cite{SemEval2018Task1} with the introduced method. The previously reported state-of-the-art results in the SemEval competition are underlined. The current state-of-the-art results are boldfaced.}
\label{tab:sota}
\end{table*}
\addtolength{\tabcolsep}{1pt}

\subsection{Evaluation Methodology and Implementation Details}
\label{sec:evaluation_methodology}
We use macro averaged F1 (macro-f1), micro averaged F1 (micro-f1) and jaccard index, which are the metrics used in the SemEval competition~\cite{SemEval2018Task1}. For fairness, we optimize our network, BERT~\cite{BERT}, RoBERTa~\cite{liu2019roberta}, XLM-RoBERTa~\cite{liu2019roberta}, and \textit{Deepmoji}~\cite{deepmoji} baselines by using Tree Parzen Estimator of the Optuna library~\cite{akiba2019optuna} and choose the model with the largest validation macro-f1 score among 100 trials. For the \textit{FastText}~\cite{fasttext} baseline, we use its own hyperparameter optimization module with 130 different trials for each of the languages, i.e., 30 more trials than our optimization for our model and the \textit{Deepmoji} baseline. For the methods in the SemEval emotion classification competition, we directly use their reported scores for fairness. These methods have also followed similar approaches for hyperparameter optimization, e.g., \textit{EMA}~\cite{badaro2018ema} performs a grid search, \textit{NVIDIA}~\cite{kant2018practical} and \textit{NTUA-SLP}~\cite{baziotis2018ntua} employ Bayesian optimization in dimensional space of all the possible values.

We train our model via the Adam~\cite{kingma2014adam} optimizer. We use ekphrasis\footnote{https://github.com/cbaziotis/ekphrasis/tree/master/ekphrasis} preprocessing library to perform language-independent preprocessing of social cues such as username normalization. We use weight decay and early stopping. We stop the training when 10 epochs are exceeded without any macro-f1 improvements on the validation set. 

\subsection{Comparison with the State-of-the-Art}
\label{sec:comparison_sota}
Here, we compare our model with the state-of-the-art methods and the best models in SemEval-2018 competition in Arabic, English and Spanish languages.

To create our best model, we combine the Arabic, English and Spanish data by combining their training and validation sets. We then train our model on the combined data using the methodology described in Section~\ref{sec:evaluation_methodology}.

We use 15 different baselines to compare our method and demonstrate its effectiveness, most of which are the highest performing contenders in SemEval-2018 emotion classification competition. \textit{NTUA-SLP}~\cite{baziotis2018ntua} ranked 1st on jaccard and micro-f1 metrics for English by using a pretrained Bi-LSTM with a multi-layer self-attention mechanism. They use word2vec embeddings that are trained on 550 million tweets. The best micro-f1 score for English is achieved by \textit{psyML}~\cite{gee2018psyml}, which uses a very similar Bi-LSTM with self-attention model to \textit{NTUA-SLP}, except they utilize hierarchical clustering to group correlated emotions together and train the same model incrementally for emotions within the same cluster. \textit{NVIDIA}~\cite{kant2018practical} trains an attention-based transformer network on large scale data and finetune this model on the training set for SemEval-English before testing it, obtaining results on par with those in the competition ranking. \textit{DeepMoji} is a distant supervision based LSTM architecture and it obtains the state-of-the-art performance on many sentiment-related tasks~\cite{deepmoji}. They convert multi-label instances into separate binary tasks. We report the results of their chain-thaw approach on the English dataset. 
For Arabic, \textit{EMA}~\cite{badaro2018ema} (1st place in jaccard and micro-f1, 2nd place in macro-f1) and \textit{PARTNA} (1st place in macro-f1, 2nd place in jaccard and micro-f1) achieves the highest two ranks. \textit{EMA} employs AraVec embeddings~\cite{aravec} as features into a support vector classifier (SVC) with $L_1$ regularization. \textit{PARTNA} uses a similar support vector based model except using an additional Arabic stemmer designed for handling tweets~\cite{SemEval2018Task1}. There are also studies that perform well but are not in SemEval rankings. Among these, \textit{CA-GRU}~\cite{samy2018context} uses context information, the topic of the text in this case, as a feature by first feeding the text to a topic-detection model to obtain a vector of probability distributions over topics. \textit{HEF-DF}~\cite{alswaidan2020hybrid} is a simple neural network hybrid model obtained from concatenating human-engineered (i.e., handpicking features that represent syntactical and semantical significance) and deep features (i.e., using combinations of embeddings).
As with Arabic, two models exist for the 1st place in a metric for Spanish: \textit{MILAB\_SNU} (1st place in jaccard and micro-f1, 2nd place in macro-f1) and \textit{ELiRF-UPV}, which uses manually and automatically generated lexicon sand combines 1D CNNs with an LSTM to obtain the 1st place in macro-f1 metric~\cite{gonzalez2019elirf} (2nd place in jaccard and micro-f1). We also include \textit{Tw-StAR}~\cite{mulki2018tw} as a baseline to compare our method's multilingual performance with a standard model. \textit{Tw-StAR} uses binary relevance transformation strategy to extract term frequency-inverse document frequency (tf-idf) features for a linear support vector machine. They also experiment with combinations of 5 different preprocessing methods and reach the 3rd rank for both Arabic and Spanish datasets.
\textit{FastText} is a framework that can convert text into feature vectors by using a skip-gram model, where each word is represented as a bag of n-grams~\cite{fasttext}. \textit{FastText} contains readily extracted vectors for 157 languages. We finetune these vectors for English, Arabic and Spanish and use them on their respective SemEval datasets. We employ pretrained BERT~\cite{BERT}, RoBERTa~\cite{liu2019roberta} and XLM-RoBERTa~\cite{liu2019roberta} as additional baselines, which are described in Section~\ref{sec:relatedwork}, and finetune all parameters of them.

Table~\ref{tab:sota} presents the results of our model compared to the state-of-the-art models and the competition-winners. The models that target only a single language perform significantly better compared to the multilingual models. The only exception is our best model, which we train on three different languages' training data combined. Our best model obtains significantly better results compared to our single language (\textit{Ours-SL}) model with the same methodology and trained on each of these languages separately. Our method achieves the best score on all of the metrics in Arabic and Spanish languages. Our method achieves the best score in macro-f1 metric of the English language. In Arabic, our method achieves 4.8\% (absolute) macro-f1 improvement compared to the previous best model, which is \textit{HEF-DF}~\cite{alswaidan2020hybrid}. Our method obtains 2.3\% (absolute) micro-f1 and 0.2\% (absolute) jaccard score improvement compared to the previous best model \textit{CA-GRU}~\cite{samy2018context}. In English, our method achieves 1\% (absolute) macro-f1 improvement compared to the previous competition winner \textit{psyML}~\cite{gee2018psyml}. Our method attains 0.5\% (absolute) micro-f1 and 1.2\% (absolute) jaccard score improvements compared to the competition winner \textit{NTUA-SLP}~\cite{baziotis2018ntua}. Notice that RoBERTa~\cite{liu2019roberta} model improves over NTUA-SLP by 0.9\% (absolute) macro-f1 and 0.3 (absolute) jaccard score. In Spanish, our method achieves 9.0\% (absolute) macro-f1 improvement compared to the previous best model \textit{ELiRF-UPV}~\cite{gonzalez2019elirf}. Our method also achieves 4.8\% (absolute) micro-f1 and 1.7\% (absolute) jaccard score improvement compared to the previous best model \textit{MILAB\_SNU}.

We perform post-hoc Nemenyi test among the multilingual models (Ours, Ours-SL, Tw-StAR, FastText, XLM-RoBERTa)~\cite{demvsar2006statistical}. Note that we could not perform post-hoc Nemenyi test for all models here since not all of them report results on different datasets. We first perform Friedman test where we obtain $p < 0.05$ for each of macro-f1, micro-f1 and jaccard index metrics. Then, we proceed with post-hoc Nemenyi test. In the Nemenyi part, we have again repeated the test for each metric and obtained the same result from each: we can only reject the null hypothesis with $p < 0.05$ that results of the FastText and our best model were drawn from the same distribution. Notice that since there are only three datasets, it is not possible to reject more than one in the optimal scenario since the maximum possible average rank difference is $4$ and the critical difference is $3.52$ (which must be greater than the average rank difference in order to reject).

\subsection{Cross-Lingual Experiments}
\label{sec:cross_lingual_experiments}
In this section, we demonstrate the cross-lingual capability of our method using training and test data combinations of different languages.

Table~\ref{tab:cross_lingual} presents the results when a model is trained on combinations of the datasets of different languages from the SemEval competition~\cite{SemEval2018Task1}. For each row, we train the model using the combined training data of the languages in the first column and validate using the combined validation data. We then experiment on the test sets of the English (EN), Spanish (SP) and Arabic (AR) languages separately. Note that the threshold and the best model are selected using the validation set of the combined data using the procedure we describe in Section~\ref{sec:evaluation_methodology}. Recall that we use only a single model, which is the model shown in the last row (AR + EN + SP) in the comparisons with the state-of-the-art in Section~\ref{sec:comparison_sota}.

\begin{table}[htbp!]
    \centering
    \begin{tabular}{lccc}
    & \multicolumn{3}{c}{Validation Data} \\

  \cmidrule(lr){2-4}
        Training Data & Arabic & English & Spanish  \\     \midrule
Arabic (SP) & 52.7 & 39.6 & 30.7 \\
English (EN) & 37.9 & 60.1 & 36.2 \\
Spanish (SP) & 35.2 & 46.4 & 52.3 \\
AR + EN & 52.5 & 61.1 & 39.6 \\
AR + SP & \textbf{57.9} & 47.5 & 52.7 \\
EN + SP & 44.9 & 60.5 & \textbf{53.9} \\
EN + AR + SP & 55.3 & \textbf{61.7} & 52.6 \\
\bottomrule
    \end{tabular}
    \caption{Experiment results when the model is trained on the combinations of the SemEval datasets and tested on the individual validation sets. The best results are boldfaced.}
    \label{tab:cross_lingual}
\end{table}

\begin{figure*}[t!] 
    \centering
  \subfloat[Only English data]{%
       \includegraphics[width=0.47\linewidth]{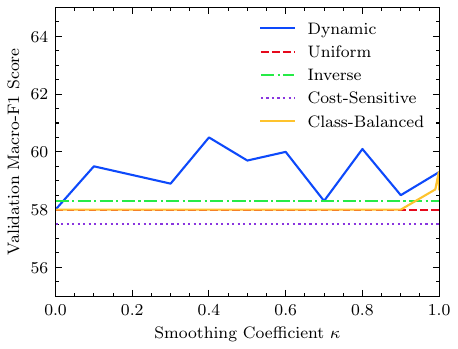}}
    \hfill
  \subfloat[Combined Arabic, English and Spanish data]{%
        \includegraphics[width=0.47\linewidth]{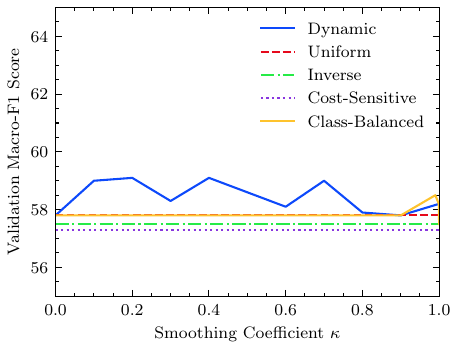}}
  \caption{Comparison of the weighting methods for imbalanced classification on the data from the SemEval emotion classification competition~\cite{SemEval2018Task1}. Figure is best viewed in color.}
  \label{fig:ablation_loss} 
\end{figure*}

As expected, the models trained on a single language perform the best on the training data's language. For instance, the model trained in English performs the best for the English test data. This is due to the semantic differences and the implicit biases in each dataset. The results clearly indicate that training with data from different languages significantly improves the performance of our model. For English test data, including Arabic to the English training data improves the model more than including Spanish. For Arabic test data, including Spanish to the Arabic training data improves the model more than including English. For Arabic test data, using Arabic and Spanish training data combined performs the best. For English test data, using data from all of the three languages performs the best. For Spanish, including English data to Spanish training data improves the model more than including Arabic. For Spanish test data, using English and Spanish training data combined performs the best and including Arabic to this data lowers the performance.

Our cross-lingual experimental results are consistent for the models that are trained on single language datasets with the semantic similarity atlas of the languages~\cite{csenel2018generating}. For example, English and Spanish are significantly more similar to each other than to the Arabic language. Among the models that are trained on a single language, English and Spanish training datasets score the best for each other's test data compared to the Arabic. For English test data, Spanish training data scores 7.0\% (absolute) macro-f1 more than the Arabic. For Spanish test data, the model trained on the English training data scores 5.5\% (absolute) macro-f1 more than the Arabic. For Arabic, which is closer to English than the Spanish in the similarity atlas~\cite{csenel2018generating}, training with English data results in 2.7\% (absolute) macro-f1 gain compared to training with Spanish. Notice that the models perform promising even for the unseen languages, e.g., the model that is trained on English and Spanish data and tested on the unseen Arabic validation data perform with 3\% to 11\% (absolute) less macro-f1 score compared to the baselines and our best model trained on all of the three languages. The model tested on the English data, which trained on the rest of the languages, perform 2\%  (absolute) macro-f1 better than the Tw-StAR baseline and  10.9\% (absolute) macro-f1 worse compared to the best English model trained on all of three languages. For the model tested on the Spanish validation data, which trained on the rest of the languages, the model performs 0.4\% (absolute) macro-f1 better than the Tw-StAR baseline and 13.4\%  (absolute) {macro-f1} worse than our best model that is trained on all of the three languages. Note that these cross-lingual scores are obtained on the unseen validation sets of the datasets to prevent the test leak, unlike the baselines, where they are tested on the test set.

\subsection{Influence of Dynamic Weighting}
\label{sec:ablation_loss}
Here, we analyze the hyperparameter selection of our dynamic weighting method and compare it with the existing weighting methods that are proposed to remedy the class imbalance.

Fig.~\ref{fig:ablation_loss} illustrates the comparison of different weighting methods in the literature and our dynamic weighting method. We use the parameters of the best model except for $\kappa$, which is the smoothing hyperparameter for dynamic weighting. For dynamic weighting method, we experiment with different $\kappa \in [0, 1]$ with 0.1 spacing. For class-balanced focal loss term, we additionally experimented with the $\beta \in \{0.99, 0.999, 0.9999\}$ values as in~\cite{cui2019class}. Note that $\beta$ is defined for $[0,1)$, thus, we did not experiment for $\beta=1$. We show the $\beta$ term of the class-balanced focal loss via the x-axis of Fig.~\ref{fig:ablation_loss}, too, which controls the growth rate of the weight with respect to the number of instances belonging to each class. We experiment with uniform weighting that assign equal importance to the losses from each class, i.e., $\alpha_{t,a}=\frac{1}{w}, \, \forall t\in \{1,2,...\}, a\in\{1,2,...w\}$. We also compare with the inverse loss, which is the inverse of the number of instances belonging to each class. Lastly, we compare with the cost-sensitive loss~\cite{aurelio2019learning}.

Our dynamic weighting method demonstrates significant performance improvement, i.e., $\approx$2.5\% (absolute) macro-f1 improvement compared to the uniform weighting and more improvements compared to the other methods on the only English data. The only exception is the class balanced weighting, for which our method achieves $\approx$1.2\% (absolute) macro-f1 improvement compared to the best of the class-balanced weighting when $\beta=0.999$. On the combined data, the dynamic weighting achieves 1.3\% macro-f1 improvement when $\kappa=0.4$ compared to the uniform weighting and more improvements compared to the other methods. The only exception is the class balanced weighting, for which our method achieves 0.6\% (absolute) macro-f1 improvement when $\beta=0.99$.

Although there exist fluctuations w.r.t. $\kappa$ hyperparameter, it performs no worse than the default uniform weighting for any of the $\kappa$ values for both only English and combined Arabic, English and Spanish data. Notice that when $\kappa=0$, the dynamic weighting method is equivalent to the uniform weighting since the $\phi$ parameter is never updated. Our method achieves its best value at $\kappa=0.4$ for both datasets.

\subsection{Ablation Study}
Here, we perform an ablation study to assess the performance gains obtained by the components in our method. We experiment with our best model, including both focal loss and dynamic weighting (Ours), our model with dynamic weighting only (Ours w/o FL), our model with focal loss only (Ours w/o DW), bidirectional recurrent neural networks (Bi-RNN)~\cite{elman}, bidirectional gated recurrent unit (Bi-GRU)~\cite{chung2014empirical}, unidirectional long short-term memory (Uni-LSTM)~\cite{hochreiter1997long}, bidirectional long short-term memory (Bi-LSTM)~\cite{schuster1997bidirectional}, and XML-CNN~\cite{liu2017deep} model that is proposed for the extreme multi-label classification tasks with more than thousand labels. Note that all models include focal loss and dynamic weighting unless otherwise stated.

\begin{table*}[htbp!]
\centering
\begin{tabular}{lcccccccccccc}
& \multicolumn{2}{c}{Precision} & \multicolumn{2}{c}{Recall} & \multicolumn{2}{c}{F1} & \multicolumn{4}{c}{Ranking Based} \\
\cmidrule(lr){2-3}
\cmidrule(lr){4-5}
\cmidrule(lr){6-7}
\cmidrule(lr){8-11}
Method & Macro & Micro & Macro & Micro & Macro & Micro & Zero-One & Coverage & RL & AP \\ \midrule
    Ours & {\bf 56.4} & 62.0 & 64.1 & {\bf 73.5} & {\bf 59.4} & {\bf 67.2} & {\bf 0.7787} & {\bf 3.51} & {\bf 0.0994} & {\bf 62.7}\\ \midrule
    Ours w/o FL & 54.2 & {\bf 62.4} & 60.6 & 69.7 & 57.1 & 65.9 & 0.7814 & 3.54 & 0.1021 & 61.1\\ 
    Ours w/o DW & 53.7 & 60.4 & 64.3 & 71.4 & 57.8 & 65.4 & 0.7814 & 3.52 & 0.0998 & 62.4\\
    Ours w/o FL+DW & 52.2 & 59.7 & {\bf 64.8} & 72.7 & 57.4 & 65.6 & 0.8020 & 3.56 & 0.1039 & 60.7\\
    Uni-LSTM & 54.1 & 60.6 & 61.7 & 71.0 & 57.1 & 65.4 & 0.7967 & 3.53 & 0.1009 & 60.4\\
    Bi-RNN & 52.4 & 59.9 & 61.8 & 70.9 & 56.4 & 65.0 & 0.8086 & 3.62 & 0.1069 & 59.1\\
    Bi-GRU & 54.5 & 61.3 & 62.6 & 72.6 & 57.7 & 66.5 & 0.7854 & 3.55 & 0.1038 & 61.9\\
    XML-CNN~\cite{liu2017deep} &  40.8 & 44.6 & 62.9 & 70.2 & 47.8 & 54.5 & 0.9176 & 4.18 & 0.1578 & 48.0\\
    
\bottomrule
\end{tabular}
\caption{Ablation study of different network architectures and losses on the validation set of the combined Arabic, English and Spanish data. RL and AP indicate label ranking loss~\cite{Tsoumakas2010} and average precision, respectively. "Zero-One" and "Coverage" indicate zero-one error and coverage error, respectively. Lower is better for zero-one error, coverage error and ranking loss, and higher is better for the rest.}
\label{tab:ablation_model}
\end{table*}

Table~\ref{tab:ablation_model} presents the results on the validation set of the combined Arabic, English and Spanish data when the recurrent component is changed with other models and the loss changed with the standard cross-entropy loss. For each row, we only change the loss or the model. We keep all other hyperparameters as is. Our model with focal loss and dynamic weighting performs significantly better compared to others and outperforms them in 8 of 10 metrics. Adding focal loss improves the model by 2.2\% (absolute) macro-f1, and adding dynamic weighting improves the model by 2.7\% (absolute) macro-f1. To further understand the effects of focal loss and dynamic weighting on all classes, we first conduct Friedman test by calculating per class F1 scores for each classifier on Table~\ref{tab:ablation_model}. With $p < 0.01$, we reject the null hypothesis that all models perform the same. Following the highly significant Friedman test, we perform The Wilcoxon signed-rank test to out model using the same per-class F1 setup. The results show us that there is no statistical significance in using focal loss or dynamic weighting by themselves ($p > 0.05$), but combining both of these methods causes a significant increase in performance ($p < 0.05$). Unidirectional LSTM, which runs on the sentences only in the forward direction, performs 2.3\% (absolute) macro-f1 worse compared to its bidirectional variant. Although the GRU works better than the RNN, it performs 1.7\% worse compared to the bidirectional LSTM. XML-CNN, which is a CNN based model, performs significantly worse compared to the other models.

\subsection{Running Time Comparisons}
\begin{table}[t!]
\centering
\begin{tabular}{lcc}
\toprule
Method & Time Per Epoch & Total Time \\ \midrule
    Ours & 25.3 secs & 5 mins 4 secs\\ \midrule
    Ours w/o FL & 24.6 secs & 4 mins 59 secs\\ 
    Ours w/o DW & 25.2 secs & 4 mins 58 secs\\
    Ours w/o FL+DW & 25.0 secs & 5 mins 1 secs\\
    Uni-LSTM & 17.6 secs  & 3 mins 10 secs \\
    Bi-RNN & 18.4 secs & 4 mins 7 secs\\
    Bi-GRU & 21.5 secs & 3 mins 28 secs \\
    XML-CNN~\cite{liu2017deep} & 17.1 secs & 3 mins 30 secs \\
    XLM-RoBERTa~\cite{conneau2019unsupervised} & 306.2 secs & 81 mins 3 secs \\
\bottomrule
\end{tabular}
\caption{Running time measurements of different models. Time per epoch depicts the time it takes for the model to apply forward and backward operations for each training data once. Total Time is the total elapsed time of execution starting from the data read operation and ending with the converge of the model, which can happen with early stopping.}
\label{tab:runtime}
\end{table}
Table~\ref{tab:runtime} shows elapsed times per epoch and in total. We have performed measurements through NVIDIA GeForce GTX 1080 Ti. Our best model, containing all components, takes about 5 minutes to run. The methods we introduced, i.e., dynamic weighting and focal loss for multi-label classification, incur negligible running time cost. The increase in time per epoch is no more than 0.7 seconds, and the increase in total time is no more than 5 seconds. Moreover, these methods do not affect the asymptotic time complexity for Big-O notation.

Furthermore, XML-CNN, Bi-RNN and Uni-LSTM reduce the running time by 8 seconds at most per epoch compared to our best model. These models also reduce the total time by 1 minute and 56 seconds at most compared to our best model. However, our best model is significantly more accurate than other models, as shown in Table~\ref{tab:ablation_model}. Notice that XLM-RoBERTa is considerably slower due to the number of trained parameters.

\subsection{Individual Class Performances}
In this section, we analyze the performance of our method for individual classes.

\begin{figure}[htbp!]
    \centering
    \includegraphics[height=6cm,keepaspectratio]{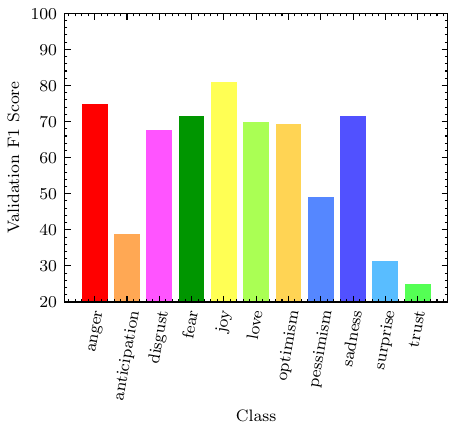}
    \caption{Per-class F1 scores of the validation set of the combined Arabic, English and Spanish data. Figure is best viewed in color.}
    \label{fig:perclass_barplot}
\end{figure}

Fig.~\ref{fig:perclass_barplot} illustrates the validation F1 score for all classes on the combined data using the best model on combined data obtained in Section~\ref{sec:comparison_sota}. The model performs the best for the {\it joy} class with 80.9\% macro-f1 and performs the worst for the {\it trust} class with 25.0\% macro-f1. The {\it surprise} and the {\it trust} classes perform the worst among all as expected since their number of instances is the least. Discrimination of the {\it optimism} is significantly better than the {\it pessimism} as the number of instances in the {\it optimism} class is significantly higher than the number of instances in the {\it pessimism} class. Interestingly, the {\it anticipation} class is the third worst performing class, although it is not the third in terms of rarity, which is consistent with the results of the NVIDIA study~\cite{kant2018practical}. Our model performs around 70\% for the rest of the classes, i.e., {\it anger}, {\it disgust}, {\it fear}, {\it joy}, {\it love}, {\it optimism} and {\it sadness}. In the following section, we demonstrate the performance improvements obtained by the components in our method.
\section{Conclusion}
\label{sec:conclusion}

We have investigated the cross-lingual sentiment analysis in multi-label setting. We have introduced a system that performs sentiment analysis in 100 different languages. To cope with the inherent class imbalance problem of multi-label classification, we have introduced a dynamic weighting method to remedy the inherent class imbalance problem of multi-label classification. This method balances the loss contribution of the classes as the training progresses, unlike the static weighting methods that assign non-changing weights to the classes. We have adapted the focal loss to the multi-label setting from the single-label object recognition literature.
Moreover, we have derived a macro-f1 maximization method in linear time complexity for choosing class-specific thresholds to produce predictions. Our system has achieved the state-of-the-art performance in 7 out of 9 metrics in 3 different languages on the SemEval emotion classification competition~\cite{SemEval2018Task1}. We have demonstrated the performance gains compared to the first ranking methods in the SemEval emotion classification competition~\cite{SemEval2018Task1} and the common baselines. We have also evaluated our method in the cross-lingual setting. We have demonstrated the performance gains obtained by the dynamic weighting and analyzed the effects of our method's components through an ablation study.

\bibliography{ref.bib}
\balance
\bibliographystyle{IEEEtran}

\end{document}